\documentclass[journal]{IEEEtran}
\ifCLASSOPTIONcompsoc
  \usepackage[nocompress]{cite}
\else
  \usepackage{cite}
\fi

\usepackage[dvips]{graphicx}
\usepackage{array}
\usepackage{threeparttable}
\usepackage{amsmath}
\usepackage{multirow}
\usepackage{booktabs}
\usepackage{graphicx}
\usepackage{tabularx}
\usepackage{float}
\usepackage{amsopn}
\usepackage{xcolor}
\usepackage{makecell}
\usepackage{caption}
\usepackage[ruled,norelsize]{algorithm2e}
\usepackage{stfloats} 

\usepackage{subcaption,graphicx}

\makeatletter
\newcommand{\removelatexerror}{\let\@latex@error\@gobble}
\makeatother

\newcounter{appcount}
\newcommand{\appendicesname}
            {Appendix\ \thechapter  \Alph{appcount}}
\newcommand{\bookappendicesname}
            {Appendix\ \Alph{appcount}}
\newcommand{\chapterappendix}[1]
          {\par\setcounter{section}{0}
           \setcounter{equation}{0}
           \setcounter{table}{0}
           \setcounter{figure}{0}
          \addtocounter{appcount}{1}   \renewcommand{\theequation}{\thechapter\Alph{appcount}.\arabic{equation}}
          \renewcommand{\thetable}{\thechapter\Alph{appcount}.\arabic{table}}
          \renewcommand{\thefigure}{\thechapter\Alph{appcount}.\arabic{figure}}
           \setcounter{section}{\arabic{chapter}\Alph{section}}
           \if@openright\cleardoublepage\else\clearpage\fi
           \chapter*{\huge{\appendicesname}\newline\newline \Huge{#1}}
           \addcontentsline{toc}{section}{\thechapter\Alph{appcount} #1}
           \markright{\MakeUppercase{\appendicesname.\ { #1}}}}
\newcommand{\bookappendix}[1]
          {\par\setcounter{section}{0}
           \setcounter{equation}{0}
           \setcounter{table}{0}
           \setcounter{figure}{0}
          \addtocounter{appcount}{1}   \renewcommand{\theequation}{\Alph{appcount}.\arabic{equation}}
           \renewcommand{\thetable}{\Alph{appcount}.\arabic{table}}
             \renewcommand{\thefigure}{\Alph{appcount}.\arabic{figure}}
           \setcounter{section}{\arabic{chapter}\Alph{section}}
           \if@openright\cleardoublepage\else\clearpage\fi
           \chapter*{\huge{\bookappendicesname}\newline\newline \Huge{#1}}
           \addcontentsline{toc}{chapter}{\bookappendicesname #1}
          \markright{\MakeUppercase{\bookappendicesname.\ { #1}}} }
\newcounter{example}
\newcounter{property}
\newcommand{\ben}{\begin{equation}}
\newcommand{\een}{\end{equation}}
\newcommand{\bea}{\begin{eqnarray*}}
\newcommand{\eea}{\end{eqnarray*}}
\newcommand{\bean}{\begin{eqnarray}}
\newcommand{\eean}{\end{eqnarray}}

\newcommand*{\Scale}[2][4]{\scalebox{#1}{$#2$}}%

\DeclareMathOperator*{\argmax}{arg\,max}

\begin{document}

\title{EEF: Exponentially Embedded Families with Class-Specific Features for Classification}

\author{Bo~Tang,~\IEEEmembership{Student Member,~IEEE,}~Steven Kay,~\IEEEmembership{Fellow,~IEEE,}~Haibo~He,~\IEEEmembership{Senior Member,~IEEE,}~and~Paul M. Baggenstoss,~\IEEEmembership{Senior Member,~IEEE}% <-this % stops a space
\IEEEcompsocitemizethanks{\IEEEcompsocthanksitem Copyright (c) 2015 IEEE. Personal use of this material is permitted. However, permission to use this material for any other purposes must be obtained from the IEEE by sending a request to pubs-permissions@ieee.org.
\IEEEcompsocthanksitem This work was supported in part by National Science Foundation (NSF) under grant ECCS 1053717 and CCF 1439011, and the Army Research Office under grant W911NF-12-1-0378.
\IEEEcompsocthanksitem Bo Tang, Steven Kay, and Haibo He are with the Department of Electrical, Computer and Biomedical Engineering at the University of Rhode Island, Kingston, RI, USA, 02881. E-mail: \{btang, kay, he\}@ele.uri.edu}
\IEEEcompsocitemizethanks{\IEEEcompsocthanksitem Paul M. Baggenstoss is with Fraunhofer FKIE, Wachtberg, Germany, 53343. E-mail: paul.m.baggenstoss@ieee.org.}
}

\IEEEcompsoctitleabstractindextext{
\begin{abstract}
In this letter, we present a novel exponentially embedded families (EEF) based classification method, in which the probability density function (PDF) on raw data is estimated from the PDF on features. With the PDF construction, we show that class-specific features can be used in the proposed classification method, instead of a common feature subset for all classes as used in conventional approaches. We apply the proposed EEF classifier for text categorization as a case study and derive an optimal Bayesian classification rule with class-specific feature selection based on the Information Gain (IG) score. The promising performance on real-life data sets demonstrates the effectiveness of the proposed approach and indicates its wide potential applications.
\end{abstract}

}

\maketitle

\IEEEdisplaynotcompsoctitleabstractindextext

\IEEEpeerreviewmaketitle

\section{Introduction}
Classification is one of fundamental problems in the fields of machine learning and signal processing. The commonly used classifier assigns a sample or a signal to the class with maximum posterior probability, which usually requires probability density function (PDF) estimation in an either model-driven or data-driven manner \cite{duda2012pattern}\cite{bishop2006pattern}\cite{bo_enn}. For high-dimensional data sets, it is necessary to perform feature reduction to estimate the PDFs robustly in a low-dimensional feature subspace. However, feature reduction may lose pertinent information for discrimination. For example, data samples from different classes that could be well separated in the raw data space may be overlapped in the feature subspace, causing classification errors. 

The PDF reconstruction approach provides a solution to address this information loss issue in feature reduction by reconstructing the PDF on raw data and making classification in raw data space, which could improve classification performance. 
%Since this approach performs classification in the raw data space, it maintains the distributions of features of the different classes. In other words, each class could have its own feature transformation function, leading to class-specific features. Class-specific features offer many advantages for multi-class classification. For example, class-specific features carry much more discriminative information from the original raw data, because each class can select the most discrminative features against the other classes. 
Several approaches have been developed along this track. Moghaddam et al. \cite{moghaddam1997probabilistic}\cite{moghaddam2000bayesian} use an eigenspace decomposition to approximate the high-dimensional raw data PDF, where the raw data space is divided into two complementary subspaces using Principal Components Analysis (PCA): the principal subspace (distance in feature space) and the orthogonal complement subspace (distance from feature space). 
While the PDF in the low-dimensional principal subspace is estimated using training data, the PDF in the complementary subspace is approximated with the PCA residual error. Then, the estimated PDF in the raw data space is written as the product of these two PDFs. More recently, researchers apply Bayesian partitioning techniques to estimate the distribution in high dimensional data space. For example, Wong and Ma in \cite{wong2010optional} developed the Optional Polya Tree (OPT) to construct a prior distribution, and Lu et al. in \cite{lu2013multivariate} derived a closed form of posterior probability using Bayesian sequential partitioning. 

%To perform classification with class-specific features, a one-vs-all classification scheme \cite{rifkin2004defense} can be used to build hierarchical multiclassifiers \cite{kumar2000hierarchical}\cite{de2011class}. 
PDF Projection Theorem (PPT) \cite{baggenstoss2004class} offers another solution for distribution construction which projects the PDF in the feature subspace back to the raw data space. It can be shown that all PDFs that generate the given feature PDF can be constructed with the PPT by selecting a suitable reference hypothesis.
%While there may be many PDFs in raw data space that are able to produce the PDF in feature subspace exactly, the PPT finds the one that is the closest to the reference PDF. 
The generality of PPT makes it a good one for classification to avoid the ``curse of dimensionality" \cite{baggenstoss2004class}\cite{tang2016tkde}. It also allows class-specific features, that is to say, each class could have its own feature transformation function. Class-specific features offer many advantages for multi-class classification. For example, class-specific features carry much more discriminative information from the original raw data, because each class can select the most discriminative features against the other classes. This characteristic makes the PPT different from many other classifications methods which usually need to incorporate a one-vs-all classification scheme \cite{rifkin2004defense} to build hierarchical multiclassifiers \cite{kumar2000hierarchical}\cite{de2011class} to use class-specific features.
 
The exponentially embedded family (EEF) \cite{kay2005exponentially} is related to PPT. Like PPT, EEF is based on the estimated feature PDF and a specified reference hypothesis. While PPT produces a raw data PDF that reproduces the given feature PDF exactly, EEF is a way of combining one or more PDFs constructed with PPT in a geometric mixture with the reference hypothesis. The raw data PDF constructed using EEF reproduces the moments of a log-likelihood ratio statistic. This statistic can be easily estimated in the feature space and is directly linked to class separability. Thus, while PPT could be preferred for general PDF estimation, produces PDFs that are easily sampled, and offers maximum entropy optimality \cite{paul2015maximum}, EEF could be preferable in classifier design since it directly targets class separability.

%In this letter, we introduce the exponentially embedded family (EEF) which, like the PPT, constructs a PDF on the raw data space with the help of a reference hypothesis. But the requirement that PPT reproduces the given feature PDF exactly might be too restrictive. EEF relaxes this restriction and forms a geometric mixture of two PDFs: the reference hypothesis $c_0$ and the PPT density for that $c_0$.  The optimal value of the mixing weight (i.e. the maximum likelihood estimate (MLE) of the mixing parameter) can be easily found and is based on a sufficient statistic formed from the likelihood ratio in the feature space between the given class and the reference hypothesis. We show how the EEF can attain even higher classification performance than the PPT through this technique. 
%%In this letter, we present a new PDF estimation method based on the exponentially embedded family (EEF), which constructs the PDF on raw data from the PDF on features. Specifically, the constructed PDF on raw data is based on the reference distribution and the sufficient statistic of log-likelihood ratio on features. 
%Using the constructed PDF on raw data, we derive a Bayesian classifier with class-specific features, termed EEF classifier, and apply it for text categorization as a case study. The experiment results on real-life benchmarks show superior classification performance of the proposed EEF classifier and further indicate many potential applications for machine learning and signal processing. 

In this letter, we apply EEF to the class-specific classification problem and show that EEF can attain even higher classification performance than PPT. Using the constructed PDF on raw data, we derive a Bayesian classifier with class-specific features, termed EEF classifier, and apply it for text categorization as a case study. The experimental results on real-life benchmarks show superior classification performance of the proposed EEF classifier and further indicate many potential applications for machine learning and signal processing. 

%\textbf{The motivation of EEF is the following. First, we define a reference hypothesis $c_0$ that is the union of all classes. Therefore, $c_0$ is smooth (since it can be estimated from more training data) and non-committal with respect to each class. Next, we define a statistic of log-likelihood ratio $T(\mathbf{x})=\log {p(\mathbf{z}|c_i)/p(\mathbf{z}|{c_0} )}$, which is a measure of the discriminative power between the given class and the reference hypothesis. EEF then constructs a geometric mixture distribution in the raw data between the PPT density and the reference hypothesis, with the constraint that the theoretical moments of $T(\mathbf{x})$ match the observed values. This has a smoothing effect because the constructed PDF is as close as possible to the  smooth and non-committal $c_0$, yet has the same specified discriminative power as the PPT density with respect to the other classes.  We show how the EEF can attain even higher classification performance than the PPT through this technique.}

\section{EEF Classifier with Class-Specific Features}
\subsection{Background: Bayesian Classifier with Feature Reduction}
Considering a $N$-class classification task in which a data sample $\mathbf{x}$, $\mathbf{x}\in\mathcal{R}^D$, is to be classified into one of $N$ classes: $c_i, i=1,2,\cdots,N$. The optimal Bayesian classifier with minimum probability of error for this task is the \textit{maximum a posteriori} (MAP) rule which assigns class $c^*$ to $\mathbf{x}$ with a maximum posterior probability:
\begin{align}
\label{map_raw}
c^*  = \argmax_{i\in \{1,2,\cdots,N\} } p(c_i | \mathbf{x}) = \argmax_{i\in \{1,2,\cdots,N\} } p( \mathbf{x} | c_i) p(c_i)
\end{align}
where $p( \mathbf{x} | c_i)$ is the likelihood of observing $\mathbf{x}$ in class $c_i$, and $p(c_i)$ is the prior probability of class $c_i$. Usually the class-wise distribution $p( \mathbf{x} | c_i)$ is unknown and needs to be estimated from training data. For high-dimensional data, it is impractical to estimate $p( \mathbf{x} | c_i)$ accurately when the given training data is limited. For this case, one usually reduces the sample  $\mathbf{x}$ via feature transformation: $\mathbf{z} = f(\mathbf{x})$, where $\mathbf{z}\in\mathcal{R}^K$ is called the feature of $\mathbf{x}$ and the dimension of $\mathbf{z}$ is far less than that of $\mathbf{x}$, i.e., $K \ll D$. By doing so, the estimation of $p(\mathbf{z} | c_i)$ in the feature subspace is simplified. Using the MAP rule in the feature subspace, we have:
\begin{align}
\label{map_feature}
c^* = \argmax_{i\in \{1,2,\cdots,N\} } p( \mathbf{z} | c_i) p(c_i)
\end{align} 
%To ensure the above classification rule is also optimal, i.e., Eq. (\ref{map_feature}) is equivalent to Eq. (\ref{map_raw}), one has to carefully select a proper feature transformation, especially critical for various high-dimensional data analysis.
This feature-based Bayesian classifier approach forces one to make the choice between (a) sufficient feature information, but too high dimension, or (b) manageable feature dimension, but insufficient feature information. This means that there is no possiblity that Eq. (\ref{map_feature}) is equivalent to Eq. (\ref{map_raw}).  We seek to avoid this compromise by working in the raw data space and estimating $p(\mathbf{x}|c_i)$ without incurring the dimensionality problem caused by the need for a common feature set.

\subsection{EEF for PDF Construction}
%In PPT, the reconstructed PDF $p(\mathbf{x} | c_i)$ that can produce the feature PDF $p(\mathbf{z}_i | c_i)$ has minimum distance to the reference distribution $p(\mathbf{x}|c_0)$, indicating that it is critical to select an appropriate reference distribution. The inappropriate reference distribution may lead to a poor reconstruction of $p(\mathbf{x} | c_i)$. Furthermore, although PPT incorporates the distribution $p(\mathbf{z}|c_i)$ on features, it does not consider how well the reconstructed PDF $p(\mathbf{x} | c_i)$ fits the original raw data.
%
%To address these issues, we generalize the PDF projection theorem with an embedding parameter using exponentially embedded families (EEF) \cite{kay2005exponentially}. 
In this subsection, we show that the raw data PDF $p(\mathbf{x}|c_i)$ can be constructed from the feature PDF $p(\mathbf{z}|c_i)$ using EEF. First, we define a smoothing reference hypothesis $c_0$ (e.g., the union of all classes is used as $c_0$ in our study case), which is non-committal with respect to the $N$ classes. Next we define a log-likelihood ratio statistic $T(\mathbf{x}) = \log {p(f(\mathbf{x})|c_i)/p(f(\mathbf{x})|{c_0} )} = \log {p(\mathbf{z}|c_i)/p(\mathbf{z}|{c_0} )}$, which is a measure of the discriminative power between the given class and the reference hypothesis. 
%EEF then constructs a geometric mixture distribution in the raw data between the PPT density and the reference hypothesis, with the constraint that the theoretical moments of $T(\mathbf{x})$ match the observed values. This has a smoothing effect because the constructed PDF is as close as possible to the  smooth and non-committal $c_0$, yet has the same specified discriminative power as the PPT density with respect to the other classes.

Mathematically, using EEF \cite{kay2005exponentially}\cite{kay_sseef}, we estimate the PDF $p(\mathbf{x} | c_i)$ for class $c_i$ in raw data space as follows:
\begin{align}
\label{GPPT_EEF} 
p(\mathbf{x} | c_i; \theta) = \exp \left( \theta \ln \frac{p(\mathbf{z} | c_i)}{p(\mathbf{z}|c_0)} - K_0(\theta) + \ln p(\mathbf{x}|c_0)  \right)
\end{align}
where $\theta$ is the embedding parameter, and $ K_0(\theta)$ is the cumulant generating function, which is given by:
\begin{align}
\label{GPPT_EEF_K0}
K_0(\theta) & = \ln \int_{\mathbf{x}} \exp \left( \theta \ln \frac{p(\mathbf{z} | c_i)}{p(\mathbf{z}|c_0)} \right) p(\mathbf{x}|c_0) d \mathbf{x} \nonumber \\
& = \ln E_{p_{0}}\left[ \exp \left( \theta \ln \frac{p(\mathbf{z} | c_i)}{p(\mathbf{z}|c_0)} \right) \right]
\end{align}
where $E_{p_0}[\cdot]$ denotes the expectation with respect to the distribution $p_0 = p(\mathbf{x}|c_0)$. Note that for $\theta=1$, we have $K_0(\theta) = 0$ and $p(\mathbf{x}|c_i)=p(\mathbf{x} | c_0)/ p(\mathbf{z}|c_0)  p(\mathbf{z}|c_i)$, which is the PPT \cite{baggenstoss2004class}.
%The constructed PDF in Eq. (\ref{GPPT_EEF}) has a form of EEF  when the log-likelihood ratio $\ln [{p(\mathbf{z} | c_i)}/{p(\mathbf{z}|c_0)}] = T(x)$ is considered as a measurable statistic of $\mathbf{x}$.

As discussed before, the motivation of PDF construction in Eq. (\ref{GPPT_EEF}) is to effectively smooth the constructed density by minimizing the KL-divergence from $p(\mathbf{x} | c_i; \theta)$ to the smoothed and non-committal reference PDF $p(\mathbf{x}|c_0)$ with the constraints of moment-matching for the statistic $T(\mathbf{x}) = \ln [{p(\mathbf{z} | c_i)}/{p(\mathbf{z}|c_0)}]$, i.e., $E_{\hat{p}}[T(\mathbf{x})] = E_{p}[T(\mathbf{x})]$, where $\hat{p}$ denotes the PDF $p(\mathbf{x} | c_i; \theta)$ in Eq. (\ref{GPPT_EEF}) and $p$ denotes the true PDF $p(\mathbf{x} | c_i)$. The following theorem \cite{kullback1997information} demonstrates our motivation:

%The motivation of PDF construction in Eq. (\ref{GPPT_EEF}) is to minimize the KL-divergence of the constructed PDF $p(\mathbf{x} | c_i; \theta)$ to the reference PDF $p(\mathbf{x}|c_0)$ with the constraints of moment-matching for the statistic $T(\mathbf{x}) = \ln [{p(\mathbf{z} | c_i)}/{p(\mathbf{z}|c_0)}]$, i.e., $E_{\hat{p}}[T(\mathbf{x})] = E_{p}[T(\mathbf{x})]$, where $\hat{p}$ denotes the PDF $p(\mathbf{x} | c_i; \theta)$ in Eq. (\ref{GPPT_EEF}) and $p$ denotes the true PDF $p(\mathbf{x} | c_i)$. The following theorem \cite{kullback1997information} demonstrates our motivation:

\newtheorem{theorem}{\bf Theorem}
\begin{theorem}
\it Let $p_0(\mathbf{x})$ be the reference distribution and $p_1(\mathbf{x})$ be the true distribution to be estimated. Given that $T(\mathbf{x})$ is a measurable statistic such that both $\lambda = \int T(\mathbf{x}) p_1(\mathbf{x})d \mathbf{x}$ and $M(\theta) = \int \exp\left( \theta T(\mathbf{x}) \right) p_0(\mathbf{x})$ exist, the estimate $\hat{p}_1(\mathbf{x})$ with minimum KL-divergence $KL(\hat{p}_1 || p_0)$ is:
\begin{align}
\hat{p}_1(\mathbf{x}; \theta) = \exp \left( \theta T(\mathbf{x}) - \ln M(\theta) + \ln p_0(\mathbf{x}) \right)
\end{align}
\end{theorem}

\begin{proof}
The proof of this theorem is given by Kullback \cite{kullback1997information}, and its applicability has also been demonstrated in our previous work \cite{kay2005exponentially}\cite{kay2015pdf}\cite{tang2015parametric}. 
%Note that the measurable statistic is: $T(x) = \ln [{p(\mathbf{z} | c_i)}/{p(\mathbf{z}|c_0)}]$ in our constructed PDF given by Eq. (\ref{GPPT_EEF}).
\end{proof}

%The idea of our generalized PDF projection theorem is to introduce embedding parameters that are able to maximize the likelihood of observations, thus reducing the estimation degradation due to the inappropriate use of a reference distribution. 
%It can be shown that when $\theta = 1$, our constructed PDF $p(\mathbf{x} | c_i; \theta)$ from Eq. (\ref{GPPT_EEF}) is equivalent to the one constructed by Baggenstoss's PPT \cite{baggenstoss2003pdf}, and when $\theta = 0$, our constructed PDF $p(\mathbf{x} | c_i; \theta)$ is the reference distribution $p(\mathbf{x} | c_0)$. 

%In EEF, it is better to choose the reference distribution that is close to the true one in the same PDF family. It has also been shown that the reference hypothesis with the union of all classes is good one, which can be considered as the geometric center of PDFs of all classes \cite{tang2015parametric}.
In EEF, it is better to choose the reference distribution that is smooth and non-committal with respect to the $N$ classes. The reference hypothesis consisting of the union of all classes is good one, and can be
considered the geometric center of PDFs of all classes \cite{tang2015parametric}. The embedding parameter $\theta$ specifies the constructed PDF that has minimum KL-divergence to the reference distribution with the constraint of moment-matching. For each class, the optimal embedding parameter $\theta_i^*$ can be estimated using the MLE criterion, which is given by:
\begin{align}
\label{GPPT_EEF_MLE}
\theta_i^* = \argmax_{\theta \in \Theta} \theta \ln \frac{p(\mathbf{z} | c_i)}{p(\mathbf{z}|c_0)} - K_0(\theta) \quad i = 1,2,\cdots, N
\end{align}
%where $\Theta$ is the parameter space in which $K(\theta_i)$ in Eq. (\ref{GPPT_EEF_K0}) exists. 
Since the cumulant generating function $K_0(\theta)$ is strictly convex and differentiable, the target function in Eq. (\ref{GPPT_EEF_MLE}) is concave and the optimal embedding parameter $\theta_i^*$ can be easily found. 

%We note here that, although a common reference distribution is used in the PDF construction given by Eq. (\ref{GPPT_EEF}), one can further apply a different reference distribution to the PDF construction of each class, which is not explored in this letter. 

\subsection{EEF for Classification with Class-Specific Features}
The PDF construction on raw data from the PDF on features allows class-specific features for classification. Let $f_i(\mathbf{x})$ be the feature transformation for class $i$, and thus we have class-specific features $\mathbf{z}_i = f_i(\mathbf{x})$ for $i=1,2,\cdots,N$. Using Eq. (\ref{GPPT_EEF}), for each class, we can always construct the PDF $p(\mathbf{x}|c_i; \theta_i)$ in raw data space from the PDF in class-specific feature space $p(\mathbf{z}_i | c_i)$.
% as follows:
%\begin{align}
%\label{GPPT_EEF_Class_Specific}
%p(\mathbf{x}|c_i; \theta_i) =  \exp \left( \theta_i \ln \frac{p(\mathbf{z}_i | c_i)}{p(\mathbf{z}_i|c_0)} - K_0(\theta_i) + \ln p(\mathbf{x}|c_0)  \right)
%\end{align}
%With the constructed PDF $p(\mathbf{x} | c_i; \theta_i)$ in Eq. (\ref{GPPT_EEF_Class_Specific}), we can apply the MAP rule to make the classification decision in the raw data space, which is given by:
Applying the MAP rule, we make classification decisions as follows:
\begin{align}
\label{map_gppt}
c^* = \argmax_{i\in \{ 1,2,\cdots,N \} } \theta_i \ln \frac{p(\mathbf{z}_i | c_i)}{p(\mathbf{z}_i|c_0)} - K_0(\theta_i) + \ln p(c_i)
\end{align}
%We note here that by using a common reference distribution in the PDF construction given by Eq. (\ref{GPPT_EEF_Class_Specific}), for which the feature PDF can be estimated, the classifier given by Eq. (\ref{map_gppt}) can be constructed without actually measuring the raw data $\mathbf{x}$. Nevertheless, the classifier is based on an implied raw-data PDF. One could apply a different reference distribution to the PDF construction of each class, which would require measuring $\mathbf{x}$, but is not explored in this letter.
We note here that by using a common reference distribution in the PDF construction, the classifier given by Eq. (\ref{map_gppt}) can be constructed without actually measuring the raw data $\mathbf{x}$. Nevertheless, Eq. (\ref{map_gppt}) is based on an implied raw data PDF. One could apply a different reference distribution to the PDF construction of each class, which would require measuring $\mathbf{x}$, but this is not explored in this letter.

%\section{Study Case: GPPT for BayesianFace}

\section{Study Case: EEF Classifier for Text Categorization}
In this section, we apply the proposed EEF classifier for text categorization in which the multinomial naive Bayes (MNB) is used as classifier. In Fig. \ref{TC_Flow_Chart}, we illustrate the difference between our EEF classifier and the conventional classifier for text categorization. Using the ``bag-of-words", a document is transformed to a real-valued vector through a dictionary that consists of all distinct words or phrases for a data set. In the real-value vector, the element denotes the occurrence of words in the document. Because of its high dimensionality, it is necessary to perform feature reduction to reduce the computational burden for training a classifier.
% although previous studies have shown that using more features may lead to better performance. 
Feature selection is a commonly used method for feature reduction in text categorization. In conventional approaches, a feature importance measurement, such as information gain (IG) \cite{yang1997comparative} or maximum discrimination (MD) \cite{bo_md}, is first employed to calculate feature importance for each individual class, and then a global function, such as sum or weighted average, is applied to rank features to select a common feature subset for all classes. In contrast, we rank features for each class and apply the class-specific features for classification. 

\begin{figure}[!ht]
\centering
\includegraphics[width = 6cm]{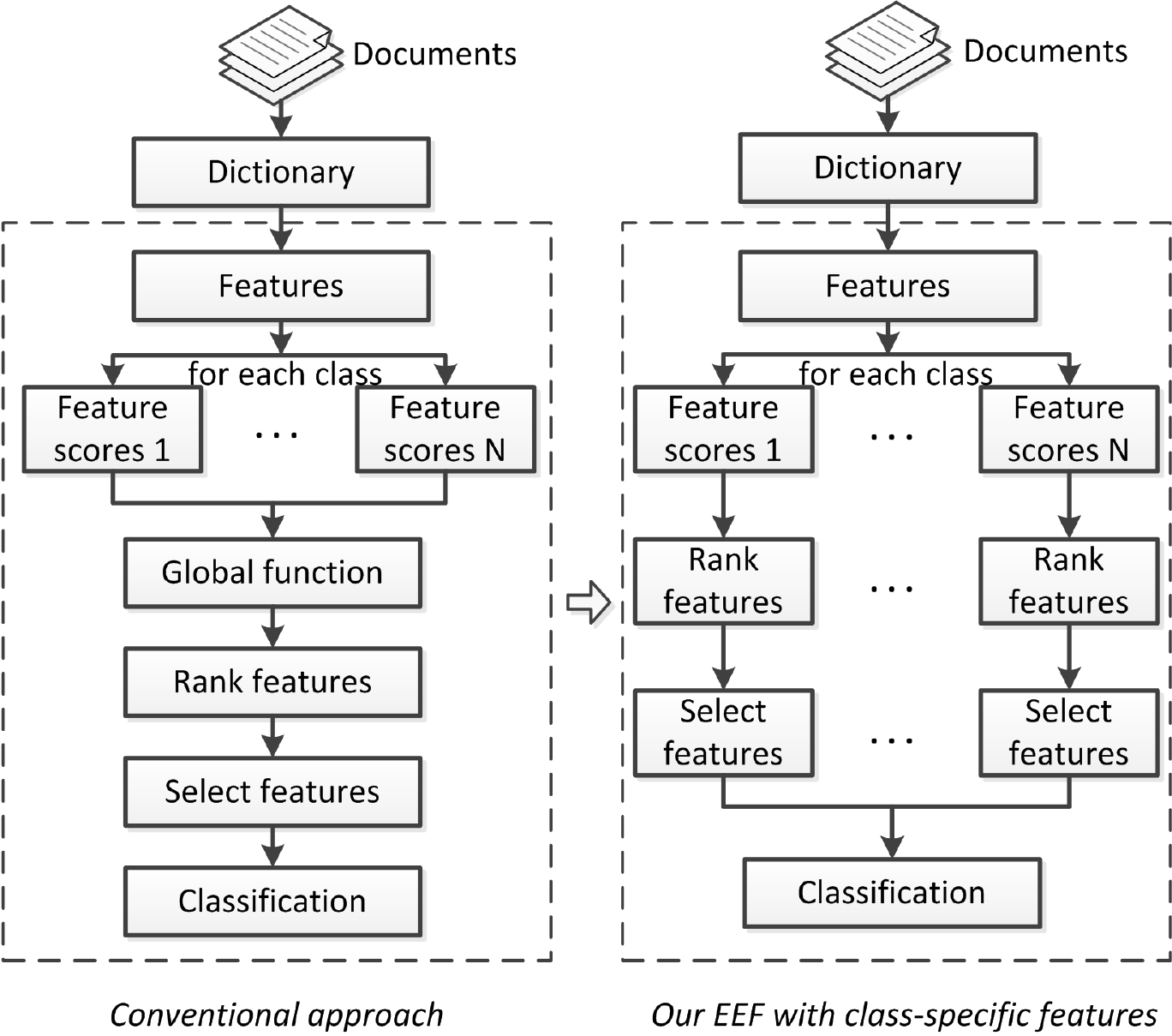}
\caption{The flow chart of our EEF classifier with class-specific features for text categorization (right), compared with the conventional approach (left).} 
\label{TC_Flow_Chart}
\end{figure}

\subsection{PDF Construction}
In MNB, the features (word occurrences) of each class satisfy a specific multinomial distribution. Let $\mathbf{x} \in \mathcal{R}^D$ be the raw feature transformed from the document, and then for each class $c_i$, $i=1,2,\cdots,N$, we have a multinomial distribution $p(\mathbf{x} | c_i)$ with $D$ parameters (cell probabilities): $[p_{i,1}, \cdots, p_{i,D}]$. The likelihood of observing a document $\mathbf{x}$ in class $c_i$ conditioned on its document length $l$ \footnote{The likelihood of observing a document $p(\mathbf{x} | c_i, l)$ is conditioned on the document length $l$. This is different than the conventional MNB classifier for text categorization in which the document length is usually assumed to be constant, i.e., $p(\mathbf{x} | c_i, l) = p(\mathbf{x} | c_i)$.}
is given by:
\begin{align}
\label{md}
p(\mathbf{x} | c_i, l) = \frac{l!}{x_1 ! x_2 ! \cdots x_D!} \prod_{k=1}^{D-1} p_{i,k}^{x_k} p_{i,D}^{x_D}
\end{align}
where $\sum_{k=1}^D p_{i,k} = 1$ and $\sum_{k=1}^D x_k = l$. 

Suppose that the feature selection will select $K$ out of $D$ features. Denote $\mathbf{z}_i$ as the feature vector in class $c_i$ and $\mathbf{I}_i = [n^i_1, \cdots, n^i_K]$ as the corresponding feature indexes in $\mathbf{x}$ such that $z_{ik} = x_{n^i_k}$. Note that the marginal distribution $p(\mathbf{z}_i | c_i)$ still satisfies a multinomial distribution, but with $K+1$ elements. The $(K+1)$-st feature is the combination of all other features in $\mathbf{x}$ except for the $K$ selected features, and the multinomial distribution $p(\mathbf{z}_i | c_i)$ has $K+1$ cells: $[p^{'}_{i,1}, \cdots, p^{'}_{i,K}, p^{'}_{i,K+1}]$ where $p^{'}_{i,k} = p_{i,n^i_k}$, $k=1,2,\cdots, K$, and $p^{'}_{i,K+1} = 1 - \sum_{k=1}^K p^{'}_{i,k}$. 

We denote class $c_0$ as the reference class which consists of all given training data so that the reference distribution $p(\mathbf{x} | c_0)$ still satisfies a multinomial distribution with $D$ parameters: $[p_{0,1}, \cdots, p_{0,D}]$, each of which can be written as:
\begin{align}
\label{mnb_reference_class}
p_{0, k} = \sum_{i=1}^N p_{i,k} p(c_i) \qquad k=1,2,\cdots, D
\end{align}

Using the general construction form in Eq. (\ref{GPPT_EEF}), we construct the PDF $p(\mathbf{x}|c_i)$ for class $c_i$, $i=1,2,\cdots,N$ as follows:
\begin{align}
\label{GPPT_EEF_MNB}
p(\mathbf{x}|c_i, l; \theta_i) 
%& = \exp \left[ \theta_i \ln \frac{p(\mathbf{z}_i | c_i )}{p(\mathbf{z}_i | c_0)} - K_0 (\theta_i, l) + \ln p(\mathbf{x} | c_0) \right ] \nonumber \\
%& 
= \exp \left[\theta_i \sum_{k=1}^{K} z_{ik} \beta_{ik} - K_1(\theta_i, l) + \ln p(\mathbf{x} | c_0) \right ]
\end{align}
where 
\begin{align}
\label{GPPT_EEF_MNB_K1}
K_1(\theta_i, l) = l \ln \left( \sum_{k=1}^K p^{'}_{0,k} \exp (\theta_i \beta_{ik}) + \left(1 - \sum_{k=1}^K p^{'}_{0,k}\right) \right)
\end{align}
and 
\begin{align}
\label{GPPT_EEF_MNB_Beta}
\beta_{ik} = \ln \frac{p^{'}_{i,k}}{p^{'}_{0,k}} - \ln \frac{p^{'}_{i,K+1}}{p^{'}_{0, K+1}}
\end{align}

%\newtheorem{lemma}{\bf Lemma}
%\begin{lemma}
%If $p_0(\mathbf{x})$ and $p_1(\mathbf{x})$ are both multinomial distributions, and $\mathbf{z}$ is a subset of $\mathbf{x}$, with given expected value $\lambda = \int \ln \frac{p_1(\mathbf{z})}{p_0(\mathbf{z})} p_1(\mathbf{x}) d \mathbf{x}$, the estimate $\hat{p}(\mathbf{x})$ with minimum KL-divergence $KL(p_1(\mathbf{x}) || p_0(\mathbf{x}))$ in GPPT is the distribution:
%\begin{align}
%\hat{p}(\mathbf{x}; \theta) & = \exp \left[ \theta \ln \frac{p_1(\mathbf{z})}{p_0(\mathbf{z})} - K_0 (\theta) + \ln p_0(\mathbf{x}) \right ] \nonumber \\
%& = \exp \left[\theta \sum_{k=1}^{K} z_{k} \beta_{k} - K_0(\theta) + \ln p_0(\mathbf{x}) \right ]
%\end{align}
%where 
%\begin{align}
%\label{beta}
%\beta_{ik} = \ln \frac{p^{'}_{i,k}}{p^{'}_{0,k}} - \ln \frac{p^{'}_{i,K+1}}{p^{'}_{0, K+1}}
%\end{align}
%and 
%\begin{align}
%K_0(\theta) = l \ln \left( \sum_{k=1}^K p^{'}_{0,k} \exp (\theta \beta_{k}) + \left(1 - \sum_{k=1}^K p^{'}_{0,k}\right) \right)
%\end{align}
%\end{lemma}

Note that we obtain a closed form solution of the PDF construction in the original high-dimensional space of $\mathbf{x}$ as shown in Eq. (\ref{GPPT_EEF_MNB}) to Eq. (\ref{GPPT_EEF_MNB_Beta}). The detailed derivation is provided in our Supplemental Material. 

Given a $N$-class training data set $\mathcal{X} = \mathcal{X}_1 \cup \mathcal{X}_2 \cup \cdots \cup \mathcal{X}_N$, each class consists of $M_i$ documents $\mathcal{X}_i = \{\mathbf{x}_1, \mathbf{x}_2, \cdots, \mathbf{x}_{M_i}\}$, and each document $\mathbf{x}_m$ has a length of $l_m = \sum_{k=1}^D x_{mk}$, where $x_{mk}$ is the $k$-th element in $\mathbf{x}_m$. We use the MLE to estimate the optimal embedding parameter, which is given by:
\begin{align}
\theta^{*}_i & = \argmax_{\theta_i \in \Theta} \theta_i \sum_{k=1}^{K} \bar{z}_{ik} \beta_{ik} - K_1(\theta_i, \bar{l})
\end{align}
where $\bar{z}_{ik}$ and $\bar{l}$ are the average of word occurrences for the $k$-th selected feature and the average of the document length over the training set $\mathcal{X}_i$ of class $c_i$, respectively. Although it is difficult to find an analytic solution of $\theta^{*}_i$, it can be easily found using convex optimization techniques since the objective function is a concave function with respect to $\theta_i$.

\section{Experimental Results and Analysis}
%\subsection{Data sets}
%In the experiments, we evaluate our GPPT based methods for text categorization on two data sets: \textsc{20 Newsgroups} and \textsc{Reuters}. 
%
%\subsection{Experimental Results}
We use two real-life data sets: \textsc{Reuters-10} and \textsc{Reuters-20}, to evaluate the performance of our proposed approach for text categorization. Both \textsc{Reuters-10} and \textsc{Reuters-20} data sets are the subsets of ModApte version of \textsc{Reuters} collection which consists of 8,293 documents with 65 classes (topics). More specifically, the data set of \textsc{Reuters-10} and \textsc{Reuters-20} consists of documents from the first 10 and 20 classes, respectively.

In these two data sets, we have an original feature size of $18,933$. To reduce the feature size, we apply the IG metric \cite{yang1997comparative} to evaluate the feature importance. For each class $c_i$, the score of the $k$-th feature is calculated as follows:
\begin{align}
\label{ig}
\Scale[1]{ IG(t_k, c_i) = p(t_k, c_i) \log \frac{p(t_k, c_i)}{p(t_k) p(c_i)} + p(\bar{t}_k, c_i) \log \frac{p(\bar{t}_k, c_i)}{p(\bar{t}_k) p(c_i)}}
\end{align}
where $t_k$ indicates the $k$-th term appears in the document, and $\bar{t}_k$ indicates it does not. It is shown that $IG(t_k, c_i)$ is a class-specific feature score. In conventional approaches, a global function, e.g., sum or average, is used to calculate class-independent feature scores for feature ranking, as shown in Fig. \ref{TC_Flow_Chart}. However, the class-specific feature based classifiers rank the feature of each class with the score $IG(t_k, c_i)$ in Eq. (\ref{ig}), and use the class-specific features for classification.

We compare our EEF class-specific MNB classifier with three other state-of-the-art classifiers: MNB classifier \cite{lewis1998naive}, support vector machine (SVM) \cite{cortes1995support}\cite{joachims1998text}, and PPT class-specific MNB classifier \cite{baggenstoss2004class}. While the first two are commonly used in text categorization with class-independent features, the last one and our classifier use class-specific features for classification. In PPT class-specific MNB classifier, we use the same reference hypothesis given by Eq. (\ref{mnb_reference_class}) and class-specific features given by Eq. (\ref{ig}) as used in EEF, and make the classification decision with the following rule:
\begin{align}
c^* = \argmax_{i=\{1,2,\cdots, N\}} \sum_{k=1}^{K+1} z_{ik} \ln \frac{p^{'}_{i,k}}{p^{'}_{0,k}} + \ln p(c_i)
\end{align}

We report the classification results on the data sets of \textsc{Reuters-10} and \textsc{Reuters-20} in Fig. \ref{reuters_ig_10} and Fig. \ref{reuters_ig_20}, respectively, where the feature size ranges from 100 to 2000. It can be shown that our EEF class-specific MNB classifier outperforms other the three methods. For the \textsc{Reuters-10} data set, the two class-specific feature based MNB classifiers greatly improve the accuracy when the feature size is small. When the feature size increases, our EEF class-specific MNB shows promising performance improvement with a large margin compared to the others. For the \textsc{Reuters-20} data set, it is seen that our EEF class-specific MNB consistently performs better than the others. 

\begin{figure}[!ht]
\centering
\includegraphics[width = 6cm]{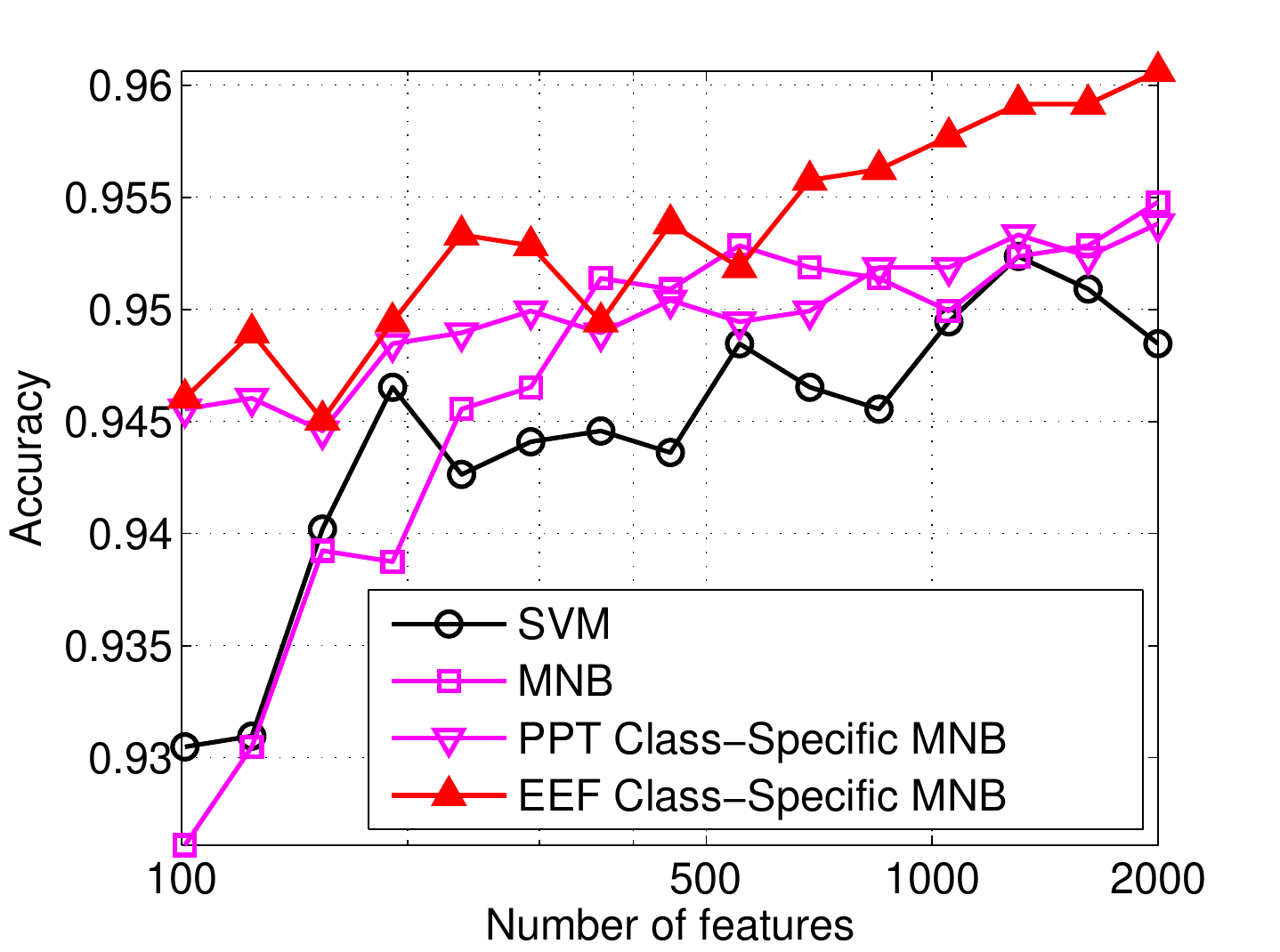}
\caption{Classification results on \textsc{Reuters-10}.} 
\label{reuters_ig_10}
\end{figure}

\begin{figure}[!ht]
\centering
\includegraphics[width = 6cm]{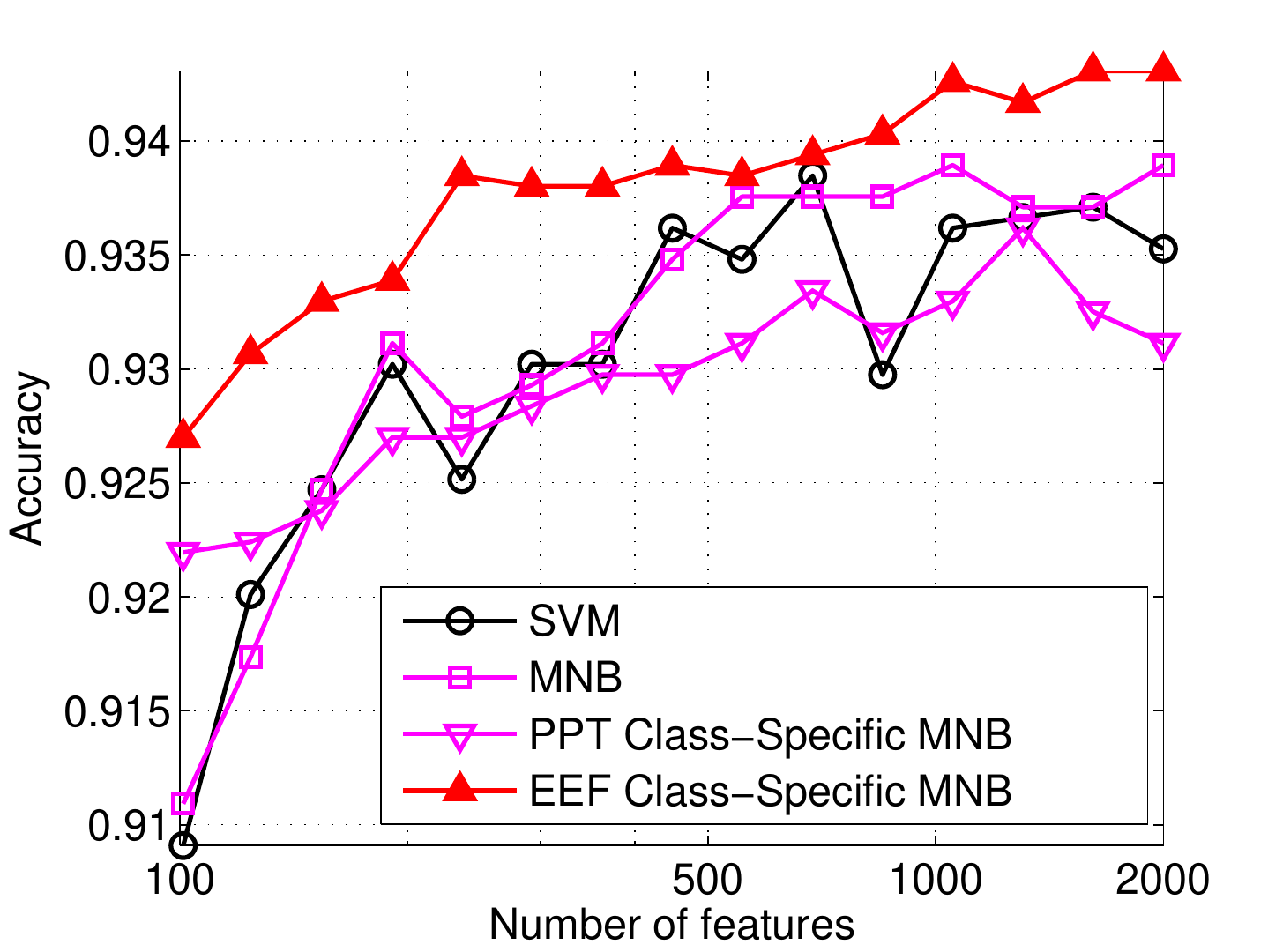}
\caption{Classification results on \textsc{Reuters-20}.} 
\label{reuters_ig_20}
\end{figure}

\section{Conclusion and Future Work}
In this letter, we introduced a new PDF construction method based on EEF to convert the feature PDF to the raw data PDF. With the constructed PDF on raw data, a Bayesian classifier with class-specific features is derived. As a case study, we applied the proposed EEF classifier for text categorization. The superior performance demonstrates the effectiveness of our proposed approach and indicates its wide potential application to machine learning and signal processing. In our future work, we will continue to explore its potential for various practical problems which might require different and complex reference distributions. Particularly, we are interested in applying sampling-based approaches to address the issue that the constructed distribution has no closed form for a complex reference distribution.

\newpage

\bibliographystyle{IEEEtran}
%\bibliography{ref}

\begin{thebibliography}{10}
\providecommand{\url}[1]{#1}
\csname url@samestyle\endcsname
\providecommand{\newblock}{\relax}
\providecommand{\bibinfo}[2]{#2}
\providecommand{\BIBentrySTDinterwordspacing}{\spaceskip=0pt\relax}
\providecommand{\BIBentryALTinterwordstretchfactor}{4}
\providecommand{\BIBentryALTinterwordspacing}{\spaceskip=\fontdimen2\font plus
\BIBentryALTinterwordstretchfactor\fontdimen3\font minus
  \fontdimen4\font\relax}
\providecommand{\BIBforeignlanguage}[2]{{%
\expandafter\ifx\csname l@#1\endcsname\relax
\typeout{** WARNING: IEEEtran.bst: No hyphenation pattern has been}%
\typeout{** loaded for the language `#1'. Using the pattern for}%
\typeout{** the default language instead.}%
\else
\language=\csname l@#1\endcsname
\fi
#2}}
\providecommand{\BIBdecl}{\relax}
\BIBdecl

\bibitem{duda2012pattern}
R.~O. Duda, P.~E. Hart, and D.~G. Stork, \emph{Pattern classification}.\hskip
  1em plus 0.5em minus 0.4em\relax John Wiley \& Sons, 2012.

\bibitem{bishop2006pattern}
C.~M. Bishop, \emph{Pattern recognition and machine learning}.\hskip 1em plus
  0.5em minus 0.4em\relax Springer, 2006.

\bibitem{bo_enn}
B.~Tang and H.~He, ``{ENN}: Extended nearest neighbor method for pattern
  recognition [research frontier],'' \emph{IEEE Computational Intelligence
  Magazine}, vol.~10, no.~3, pp. 52--60, 2015.

\bibitem{moghaddam1997probabilistic}
B.~Moghaddam and A.~Pentland, ``Probabilistic visual learning for object
  representation,'' \emph{IEEE Transactions on Pattern Analysis and Machine
  Intelligence}, vol.~19, no.~7, pp. 696--710, 1997.

\bibitem{moghaddam2000bayesian}
B.~Moghaddam, T.~Jebara, and A.~Pentland, ``Bayesian face recognition,''
  \emph{Pattern Recognition}, vol.~33, no.~11, pp. 1771--1782, 2000.

\bibitem{wong2010optional}
W.~H. Wong and L.~Ma, ``Optional p{\'o}lya tree and bayesian inference,''
  \emph{The Annals of Statistics}, vol.~38, no.~3, pp. 1433--1459, 2010.

\bibitem{lu2013multivariate}
L.~Lu, H.~Jiang, and W.~H. Wong, ``Multivariate density estimation by bayesian
  sequential partitioning,'' \emph{Journal of the American Statistical
  Association}, vol. 108, no. 504, pp. 1402--1410, 2013.

\bibitem{baggenstoss2004class}
P.~M. Baggenstoss, ``Class-specific classifier: avoiding the curse of
  dimensionality,'' \emph{IEEE Aerospace and Electronic Systems Magazine},
  vol.~19, no.~1, pp. 37--52, 2004.

\bibitem{tang2016tkde}
B.~Tang, H.~He, P.~Baggenstoss, and S.~Kay, ``A {Bayesian} classification
  approach using class-specific features for text categorization,'' \emph{IEEE
  Transactions on Knowledge and Data Engineering}, vol.~PP, no.~99, pp. 1--1,
  2016.

\bibitem{rifkin2004defense}
R.~Rifkin and A.~Klautau, ``In defense of one-vs-all classification,''
  \emph{The Journal of Machine Learning Research}, vol.~5, pp. 101--141, 2004.

\bibitem{kumar2000hierarchical}
S.~Kumar, J.~Ghosh, and M.~Crawford, ``A hierarchical multiclassifier system
  for hyperspectral data analysis,'' in \emph{Multiple Classifier Systems},
  2000, pp. 270--279.

\bibitem{de2011class}
G.~De~Lannoy, D.~Fran{\c{c}}ois, and M.~Verleysen, ``Class-specific feature
  selection for one-against-all multiclass svms,'' in \emph{European Symposium
  on Artificial Neural Networks}, 2011, pp. 269--274.

\bibitem{kay2005exponentially}
S.~Kay, ``Exponentially embedded families-new approaches to model order
  estimation,'' \emph{IEEE Transactions on Aerospace and Electronic Systems},
  vol.~41, no.~1, pp. 333--345, 2005.

\bibitem{paul2015maximum}
P.~Baggenstoss, ``A maximum entropy framework for feature inversion and a new
  class of spectral estimators,'' \emph{IEEE Transactions on Signal
  Processing}, vol.~63, no.~11, 2015.

\bibitem{kay_sseef}
S.~Kay, Q.~Ding, B.~Tang, and H.~He, ``Probability density function estimation
  using the {EEF} with application to subset/feature selection,'' \emph{IEEE
  Transactions on Signal Processing}, vol.~64, no.~3, pp. 641--651, 2016.

\bibitem{kullback1997information}
S.~Kullback, \emph{Information theory and statistics}.\hskip 1em plus 0.5em
  minus 0.4em\relax Courier Corporation, 1997.

\bibitem{kay2015pdf}
S.~Kay, Q.~Ding, B.~Tang, and H.~He, ``Probability density function estimation
  using the {EEF} with application to subset/feature selection,'' \emph{IEEE
  Transactions on Signal Processing}, vol.~64, no.~3, pp. 641--651, 2016.

\bibitem{tang2015parametric}
B.~Tang, H.~He, Q.~Ding, and S.~Kay, ``A parametric classification rule based
  on the exponentially embedded family,'' \emph{IEEE Transactions on Neural
  Networks and Learning Systems}, vol.~26, no.~2, pp. 367--377, 2015.

\bibitem{yang1997comparative}
Y.~Yang and J.~O. Pedersen, ``A comparative study on feature selection in text
  categorization,'' in \emph{International Conference on Machine Learning},
  vol.~97, 1997, pp. 412--420.

\bibitem{bo_md}
B.~Tang, S.~Kay, and H.~He, ``Toward optimal feature selection in naive bayes
  for text categorization,'' \emph{IEEE Transactions on Knowledge and Data
  Engineering}, vol.~PP, no.~99, pp. 1--1, 2016.

\bibitem{lewis1998naive}
D.~D. Lewis, ``{Naive (Bayes) at forty: The independence assumption in
  information retrieval},'' in \emph{European Conference on Machine Learning},
  1998, pp. 4--15.

\bibitem{cortes1995support}
C.~Cortes and V.~Vapnik, ``Support-vector networks,'' \emph{Machine learning},
  vol.~20, no.~3, pp. 273--297, 1995.

\bibitem{joachims1998text}
T.~Joachims, \emph{Text categorization with support vector machines: Learning
  with many relevant features}.\hskip 1em plus 0.5em minus 0.4em\relax
  Springer, 1998.

\end{thebibliography}
 \newcommand{\noop}[1]{}

\end{document}